\documentclass[copyright,creativecommons]{eptcs}
\providecommand{\event}{FMAS 2026} 

\usepackage{iftex}

\ifpdf
  \usepackage{underscore}         
  \usepackage[T1]{fontenc}        
\else
  \usepackage{breakurl}           
\fi

\usepackage{graphicx}
\usepackage{xcolor}
\usepackage{cite}
\usepackage{amsmath,amssymb,amsfonts,amsthm}
\usepackage{algorithmic}
\usepackage{algorithm2e}
\usepackage{textcomp}
\usepackage{csquotes}
\usepackage{booktabs}
\usepackage{subcaption}
\usepackage{enumitem}
\usepackage{orcidlink}

\newtheorem{theorem}{Theorem}[section]
\newtheorem{corollary}[theorem]{Corollary}
\newtheorem{lemma}[theorem]{Lemma}
\newtheorem{remark}[theorem]{Remark}

\title{No Need to Stop the Fleet: \\Localized Anomaly Isolation via Dead Zones in AGV Fleets}
\author{Natalia Ogorelysheva\orcidlink{0009-0009-7092-275X}
\institute{TU Dortmund University\\ Dortmund, Germany}
\email{natalia.ogorelysheva} \email{@tu-dortmund.de}
\and
Falk Howar\orcidlink{0000-0002-9524-4459}
\institute{TU Dortmund University\\Dortmund, Germany}
\institute{Fraunhofer Institute for Software \\ and Systems Engineering ISST \\ Dortmund, Germany}
\email{falk.howar@tu-dortmund.de}
\and
Jian-Jia Chen\orcidlink{0000-0001-8114-9760}
\institute{TU Dortmund University\\Dortmund, Germany}
\institute{RWTH Aachen University\\Aachen, Germany}
\email{jian-jia.chen@tu-dortmund.de}
}

\newcommand{\titlerunning}{No Need to Stop the Fleet}
\newcommand{\authorrunning}{N. Ogorelysheva, F. Howar \& J.-J. Chen}

\hypersetup{
  bookmarksnumbered,
  pdftitle    = {\titlerunning},
  pdfauthor   = {\authorrunning},
  pdfsubject  = {EPTCS},               
  pdfkeywords = {multi-robot coordination, path planning, automated guided vehicles (AGVs), safety guarantees, fail-operational systems} 
}

\begin{document}
\maketitle

\begin{abstract}
Automated Guided Vehicle (AGV) fleets in production and logistics environments rely on a system-wide emergency stop to handle local anomalies such as malfunctions, hazards, or contaminated areas. While safe, this approach halts the entire fleet even when only a small area is affected, causing unnecessary downtime. This paper proposes \emph{dead zones} --- a formally defined localized isolation approach that serves as an additional fail-operational mechanism for situations where a full system-wide halt is not strictly needed.
Rather than shutting down the entire fleet, a dead zone isolates the anomalous area, allowing unaffected robots to continue operating normally. 
We introduce two operational strategies: basic dead zone handling, in which affected robots enter a dead zone and stop, and dead zone handling with escaping, in which robots that are sufficiently far from a dead zone reroute around it.
For escaping on counterclockwise trajectories we prove collision-freedom with a bounded arrival delay for all escaping robots, provided a minimum pairwise separation is kept between robots. Experimental evaluation confirms the theoretical guarantees and characterizes the throughput tradeoffs between all three approaches: full emergency stop, basic dead zone handling, and dead zone handling with escaping.

\end{abstract}


\section{Introduction}
\label{sec:intro}

AGV fleets operating in production and logistics environments are subject to strict safety regulations. 
While the European standards EN~ISO~3691-4:2023 and EN~ISO~13850:2015 require an emergency stop function on every driverless industrial truck~\cite{iso3691_2023, iso13850}, in practice, a local anomaly such as a malfunctioning robot, a fallen object, or a contaminated area can trigger a system-wide halt, bringing the entire fleet to a standstill regardless of how many robots are actually affected~\cite{ogorelysheva2023mitigating,berndt2021centralized}. As fleets scale to hundreds of vehicles~\cite{liu2024efficient}, the cost of such system-wide halts grows proportionally~\cite{siemens2024downtime}.

Looking ahead, emerging applications such as automated facilities in remote or resource-constrained environments, where robots are scarce, replacement is costly, and continuous operation is critical, make this problem even more acute. In such settings, losing even a single robot to a local anomaly may be unacceptable, and a system-wide emergency stop may not be a viable option at all~\cite{zhang2023autonomous, goldsmith2022end}.

This motivates a more targeted approach: containing the anomaly within a designated area, a \textit{dead zone}, and fencing it off from the remainder of the operational space, as an additional fail-operational mechanism for cases where a full system-wide halt is not strictly necessary. Robots outside the dead zone continue operating normally; only those inside or heading towards it need to respond. Since AGV fleets typically rely on centralized offline planning~\cite{fragapane2021amr, okumura2023otimapp}, their ability to react dynamically is limited, making a pre-defined handling strategy essential. When a dead zone is announced, a robot has essentially three options depending on its position and remaining trajectory: enter the zone and stop, perform an escape maneuver, or continue if unaffected.

To the best of our knowledge, no existing work formally defines a localized isolation zone for offline-planned multi-robot systems, nor provides safety guarantees for the remaining fleet. Note, that dead zone and it's handling are distinct from general collision avoidance. In our model we assume the system is already collision-free due to used planning method. Additionally, instead of proposing a dynamic collision avoidance due to a specific anomaly, we propose an efficient dynamic rerouting based on static rule with safety guarantees.
The contributions of this paper are as follows:

\begin{enumerate}[label=(\arabic*)]
    \item We introduce the \emph{dead zone} (DZ), a formal model for localized anomaly isolation in offline-planned multi-robot systems, and provide a formalization of DZ handling based on multi-agent path finding (MAPF).
    \item We define two DZ handling strategies: Basic DZ Handling, in which affected robots enter DZ and stop there, and DZ Handling with Escaping, in which robots sufficiently far from the zone reroute around it via a counterclockwise escape strategy.
    \item We provide a formal proof of collision-freedom for the counterclockwise escape strategy, covering both cell and swap collisions, with a bounded arrival delay for all escaping robots.
    \item We evaluate all strategies experimentally, confirming the theoretical guarantees and characterizing the throughput tradeoffs between safety and operational continuity.
\end{enumerate}

\noindent\textbf{Outline.}
The remainder of this paper is structured as follows. Section~\ref{sec:motiv_example} introduces the dead zone concept through a motivational example. Section~\ref{sec:dead_zones} presents the formal model, including the system model, dead zone definition, and basic dead zone handling. Section~\ref{sec:escaping} introduces the escaping strategy. Section~\ref{sec:evaluation} describes the experimental setup and presents the results, before we discuss related works in Section~\ref{sec:related}. Section~\ref{sec:conclusion} concludes the paper with a summary and an outlook on future work.

\section{Motivational Example}
\label{sec:motiv_example}

\begin{figure}[t!]
    \centering
    \begin{subfigure}[t]{0.48\textwidth}
        \includegraphics[width=0.9\textwidth]{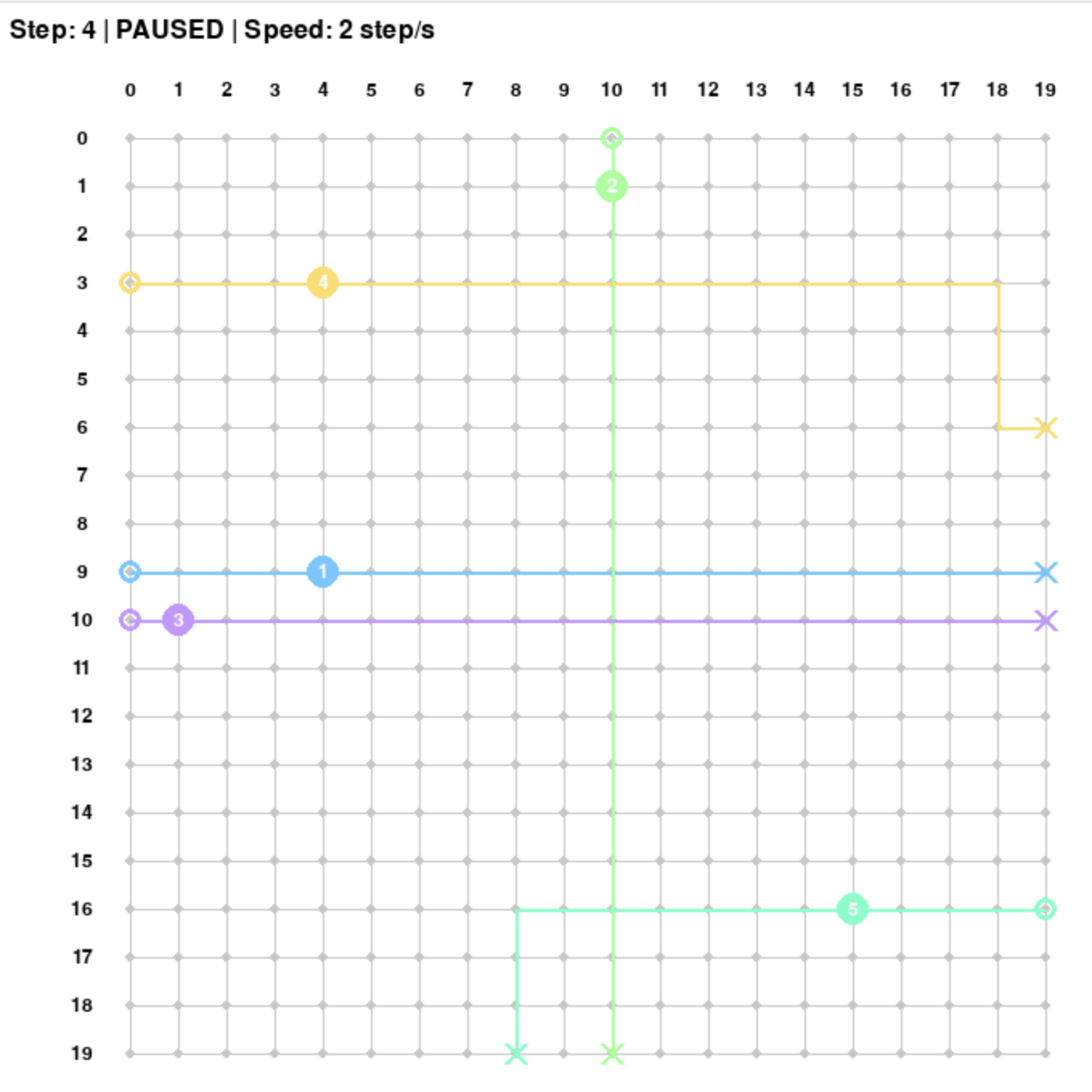}
        \caption{Normal operation ($t=4$). Robots $r_1$, $r_2$, $r_3$ have paths crossing the center region.
        All five robots moving along their planned paths without interference.}
        \label{fig:motiv_ex_robots_moving_no_DZ}
    \end{subfigure}
    \hfill
    \begin{subfigure}[t]{0.48\textwidth}
        \includegraphics[width=0.9\textwidth]{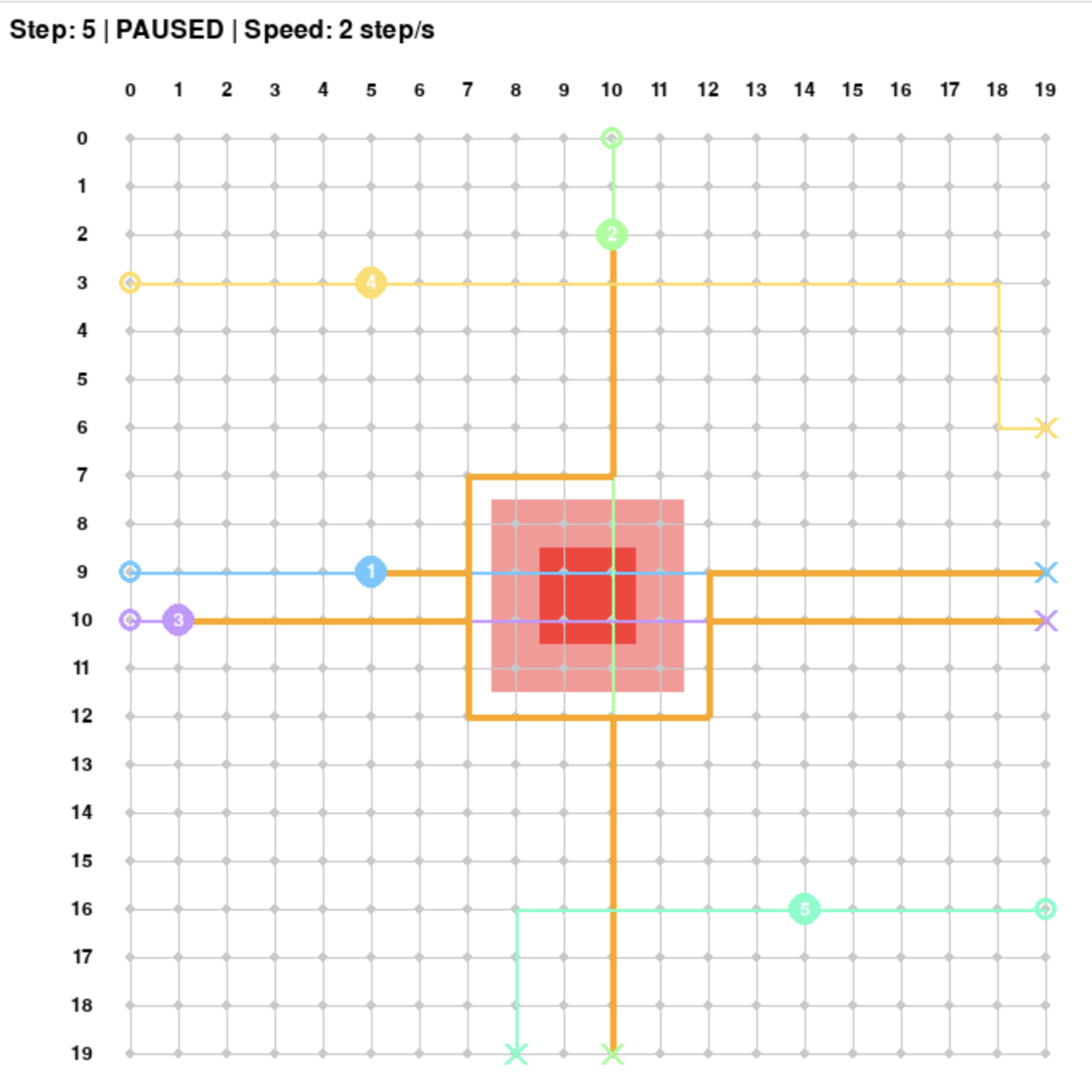}
        \caption{Dead zone announced at $t_{dz}=5$: robots $r_1$, $r_2$, $r_3$ execute counterclockwise escape around DZ (escape paths shown in orange); $r_4$, $r_5$ continue unaffected. 
        }
        \label{fig:motiv_ex_DZ}
    \end{subfigure}

    \vspace{0.5em}

    \begin{subfigure}[t]{0.48\textwidth}
        \includegraphics[width=0.9\textwidth]{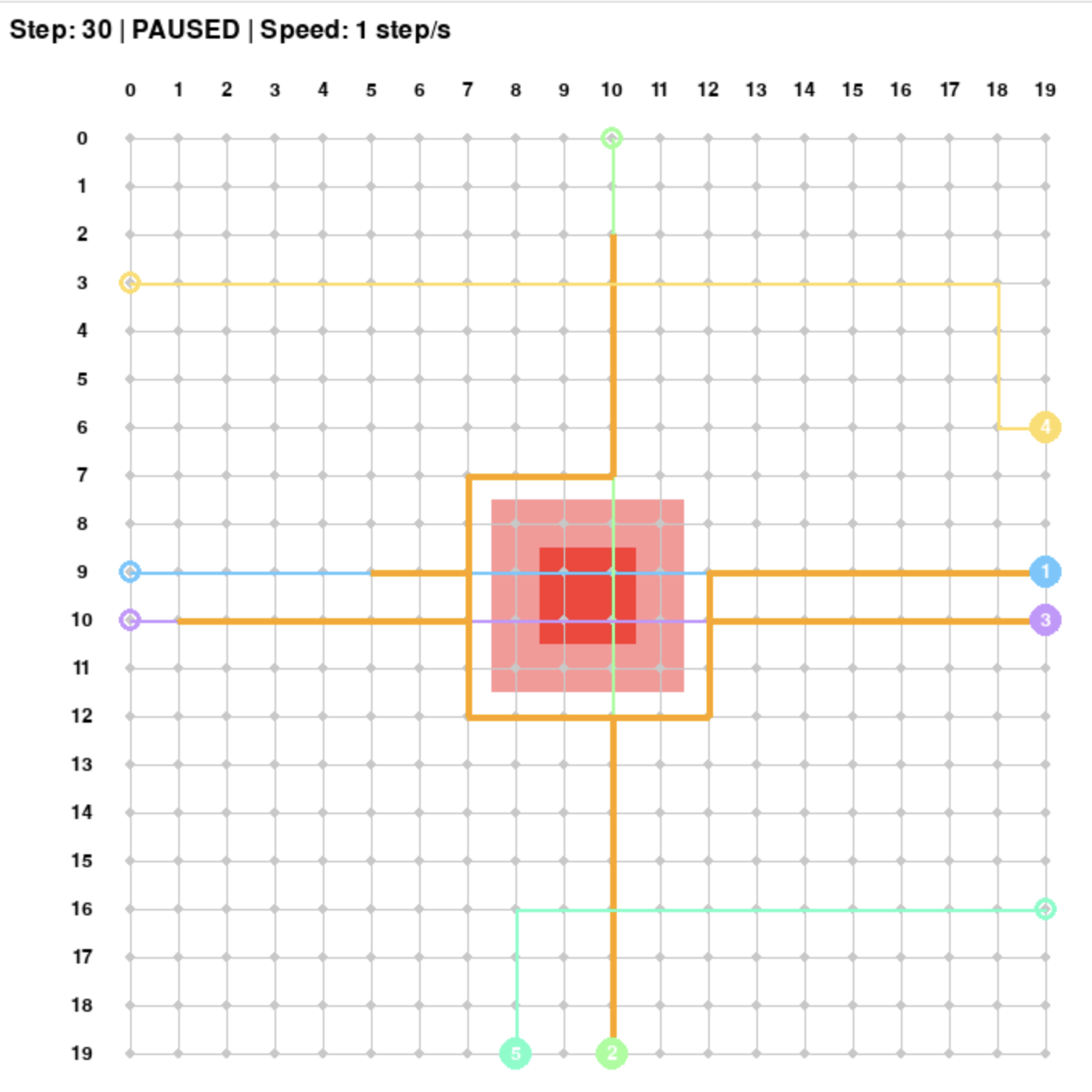}
        \caption{All robots reach their goals ($t=30$): the escape detour caused a bounded delay for $r_1$, $r_2$, $r_3$. No collisions, no losses.
         }
        \label{fig:motiv_ex_robots_finished}
    \end{subfigure}
    \hfill
    \begin{subfigure}[t]{0.48\textwidth}
        \includegraphics[width=0.9\textwidth]{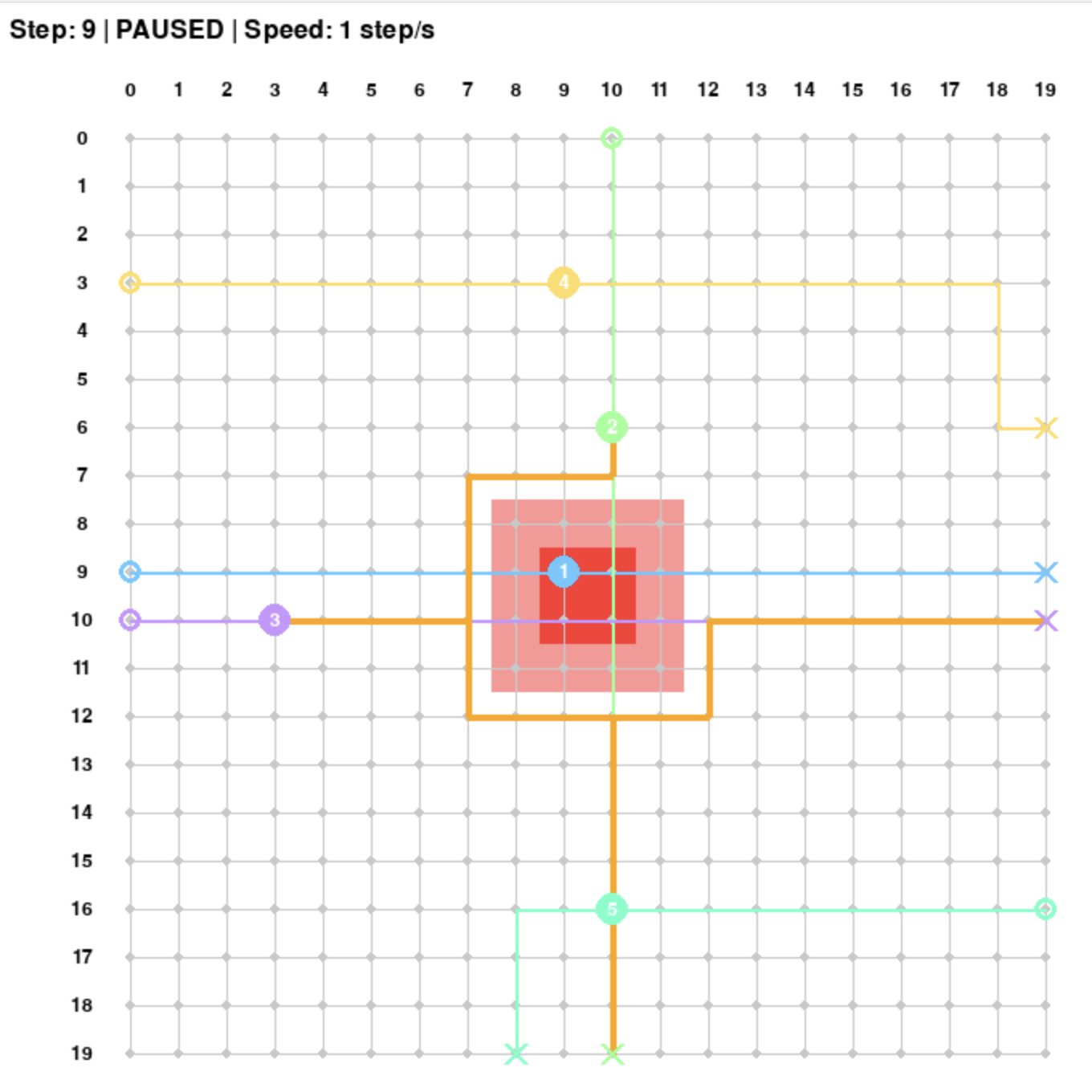}
        \caption{Late dead zone announcement ($t_{dz}=9$): robot $r_1$ is in the dead zone at announcement and considered dead. 
        }
        \label{fig:motiv_ex_alt_robot_dead}
    \end{subfigure}
    
    \caption{Five robots operating on a grid before and after a dead zone is announced. 
    }
    \label{fig:motiv_ex}
\end{figure}

We demonstrate the idea of a dead zone with a small example.
Consider a multi-robot system operating on a discrete grid (with 4-connectivity as in \cite{sharon2015cbs}), where robots continuously execute transport jobs. Fig.~\ref{fig:motiv_ex_robots_moving_no_DZ} shows five robots during normal operation at $t=4$: robots $r_1$, $r_2$, and $r_3$ follow planned paths crossing the center of the grid, while robots $r_4$ and $r_5$ operate in separate corridors well away from the center. 
All robots follow their pre-planned, collision-free trajectories without interference.

At $t_{dz} = 5$, a dead zone is announced at the center of the grid. The system must immediately decide how each robot should respond. Robots $r_4$ and $r_5$ are unaffected, their paths do not intersect the dead zone, and they continue operation without interruption. Robots $r_1$, $r_2$, and $r_3$, however, have remaining trajectories that are affected. Since all three are sufficiently far from the dead zone at $t_{dz}$, each executes an escape strategy. In this example we use one of the possible escape strategies, the counterclockwise bypassing strategy: a robot follows its original path up to the boundary of the DZ, traverses the safe perimeter counterclockwise, and rejoins its original trajectory on the other side. Fig.~\ref{fig:motiv_ex_DZ} shows the resulting escape paths immediately after the dead zone announcement.

An important requirement for any escape strategy is that it must guarantee collision-freedom --- among the escaping robots themselves, which in this case share a common resource (the escape route along the perimeter of the dead zone) and between the escaping robots and the remainder of the unaffected fleet. For instance, in Fig.~\ref{fig:motiv_ex_DZ}, robot $r_5$ continues on its original path while $r_2$ executes an escape maneuver; both share a grid cell $c=(10,16)$ at different time points. The escape strategy must ensure that such interactions remain collision-free.
Fig.~\ref{fig:motiv_ex_robots_finished} shows the system at $t=30$: all five robots have completed their jobs and reached their goal positions. The escape detour introduces a small delay for each escaping robot, but no robot is lost and no collision occurs.

For contrast, Fig.~\ref{fig:motiv_ex_alt_robot_dead} shows what happens when the dead zone is announced later, at $t_{dz}=9$. Robot $r_1$ happens to be inside the DZ at the moment of the announcement and is kept there until the anomaly is acted upon or resolved. By design the dead zone marks an area where the system can no longer guarantee safe operation. The nature of the anomaly may be entirely unknown and any robot inside it may itself be in an uncontrolled or unpredictable state. Allowing such a robot to leave the DZ and re-enter the operational area would risk introducing an uncontrolled agent into an otherwise collision-free operational area. The robot is therefore kept inside the dead zone and fenced off from the remainder of the system, ensuring that the anomaly remains localized and that the operating fleet can continue safely.

Note that the dead zone itself is independent of any particular robot response --- it merely defines the fenced area. How robots react to it, whether by entering and stopping or by escaping around it, is determined separately by the chosen handling strategy.

\section{Multi-Robot Systems, Dead Zones, and Safety}
\label{sec:dead_zones}

In this section, we formalize the concept of the dead zones.
A dead zone imposes three key constraints on the system. First, {spatial isolation}: no operating robot should enter a dead zone, and any robot already inside must remain there. Second, {operational continuity}: robots outside of a dead zone must be able to continue their tasks without interference from the anomaly. Third, {safety}: any dead zone handling strategy must maintain collision-freedom for the entire operating fleet, both among robots responding to the dead zone and between them and unaffected robots.

\noindent\textbf{Multi-robot Systems.}
We follow the standard MAPF formalism~\cite{sharon2015cbs, surynek2020bounded}. The system operates in discrete time $t \in [0, \mu]$, where $\mu$ denotes the time point at which the last robot finishes its job. The operational environment is a grid of $n \times m$ cells, where each cell $c_p$ is identified by coordinates $(u,v)$, $0 \leq u < n$, $0 \leq v < m$. Two cells are adjacent if and only if they differ by exactly 1 in exactly one coordinate (4-connectivity as in~\cite{sharon2015cbs}). The set of all cells is denoted $C$, with $\vert C\vert  = n \cdot m$.
A fleet consists of robots $R = \{r_1, \ldots, r_k\}$. 

The path $P(r_i) = [c_0^i, c_1^i, \ldots, c_+^i]$ of a robot $r_i$ is a time-independent sequence of distinct cells from start to goal, with $c_0^i$ being its start and $c_+^i$ its goal position. Applying a path planning algorithm (e.g.~\cite{sharon2015cbs}) to $P(r_i)$ yields the planned trajectory $\sigma(r_i) = [\sigma_{t_0}(r_i), \ldots, \sigma_{t_+}(r_i)]$, a time-dependent sequence that assigns each cell from start to goal a specific time point, including possible waiting steps; we write $\sigma_t(r_i) \in C$ for the planned position of $r_i$ at time $t$. Further, we denote $\sigma^{\geq t}(r_i)$ the remaining trajectory of robot $r_i$ at time point $t$. 

\noindent\textbf{Dead Zones.}
A \emph{dead zone} (DZ) is a designated area within an operational environment where the presence of robots is restricted or controlled due to safety concerns, anomalies, or potential hazards. The primary purpose of a dead zone is to protect the system and its components from further anomalies or accidents, while ensuring continued operation with limited resources. Formally, a dead zone is defined as
$DZ = \langle C_{dz},\, t_{dz} \rangle$
where $t_{dz} \in \mathbb{N}_0$ denotes the time point at which the dead zone is announced to the system, and $C_{dz} \subseteq C$ is the set of cells occupied by the dead zone, defined as the rectangular region
$C_{dz} = \{c_p \in C \mid x \leq u < x + w,\; y \leq v < y + h\}$
with $(x,y)$ denoting the top-left corner of the dead zone and $w \times h$ its dimensions. 

Note, a dead zone has no defined end time in our model yet. In practice, the resolution of anomaly may take long time, and the contaminated area might be isolated for hours. We therefore model only the announcement time $t_{dz}$ to focus on the dead zone concept itself as well as it's handling. Modeling dead zone resolution, and the eventual safe reintegration of the isolated robots into active fleet operation, once the anomaly is resolved, is a distinct problem left to future work.

After the announcement of a dead zone at $t_{dz}$, the system computes updated, or DZ-handling, trajectories $\hat{\sigma}(r_i) = [\hat{\sigma}_{t_0}(r_i), \ldots, \hat{\sigma}_{t_+}(r_i)]$ for each robot, details of which are given in this Section and in Section~\ref{sec:escaping}. We write $\hat{\sigma}_t(r_i) \in C$ for the planned position at time $t$ of the DZ-handling trajectory. 

We are interested in localized anomalies that can be contained by small dead zones and assume dead zones that cover less than $1\%$ of the operational area in this work. Larger dead zones can impose increasingly severe constraints on path planning and fleet throughput, as we analyze in Section~\ref{sec:evaluation}. 
In the extreme case where $C_{dz} = C$, the dead zone concept would reduce to a complete emergency stop of the system.

\noindent\textbf{System Safety.}
We assume that before the announcement of a dead zone, the system guarantees collision-free operation between robots on the planned trajectories. Two types of collisions are distinguished: a \emph{cell collision} occurs when two robots occupy the same cell at the same time, and a \emph{swap collision} occurs when two robots traverse the same edge in opposite directions between two consecutive time points. This fundamental safety property ensures that at any time $t \in [0, \mu]$, no two robots are planned to occupy the same cell or swap positions:
\begin{equation}
\label{eq:collision_free_assume}
\begin{split}
& \forall t \in [0, \mu], ~ \forall r_i \neq r_j \in R: \\
& \qquad \sigma_t(r_i) \neq \sigma_t(r_j) \quad \wedge \quad 
(\sigma_t(r_i) \neq \sigma_{t+1}(r_j) ~ \wedge ~ 
\sigma_{t+1}(r_i) \neq \sigma_t(r_j))
\end{split}
\end{equation}

Our goal in this paper is to develop strategies which ensure that also after replanning, no collisions occur, unless robots are inside the dead zone. 
The actual position of robot $r_i$ at time $t$ is denoted $\alpha_t(r_i)$ and satisfies:
\[
\alpha_t(r_i) =
\begin{cases}
\sigma_t(r_i), & \text{if } t \leq t_{dz},\\
\hat{\sigma}_t(r_i), & \text{if } t > t_{dz},
\end{cases}
\]
for all robots $r_i \in R$ and all time points $t$.
Our goal is then to show that
\begin{equation}
\label{eq:collision_free}
\begin{split}
& \forall t \in [0, \mu], ~ \forall r_i \neq r_j \in R: 
(t \leq t_{dz} ~ \lor ~ (\alpha_t(r_i) \notin C_{dz} \land \alpha_t(r_j) \notin C_{dz})) ~ \Rightarrow \\
& \qquad \alpha_t(r_i) \neq \alpha_t(r_j) \quad \wedge \quad 
(\alpha_t(r_i) \neq \alpha_{t+1}(r_j) ~ \wedge ~ 
\alpha_{t+1}(r_i) \neq \alpha_t(r_j))
\end{split}
\end{equation}
Another fundamental safety property of DZ handling is that, once a robot enters a dead zone, it must remain there:
\begin{equation}
\label{eq:robot_in_dz_alt}
\begin{split}
 \forall r_i \in R, ~ \forall t \in [t_{dz}, \mu]:
\quad \alpha_t(r_i) \in C_{dz} ~ \Rightarrow 
~ \alpha_{t+1}(r_i) \in C_{dz}
\end{split}
\end{equation}

The proof obligation for dead zone handling strategies, i.e., plans for $\hat{\sigma}(r_i)$ is then to show that, assuming Eq.~\ref{eq:collision_free_assume}, the definition of $\hat{\sigma}(r_i)$ together entail Eq.~\ref{eq:collision_free} and Eq.~\ref{eq:robot_in_dz_alt}. Since Eq.~\ref{eq:collision_free_assume} covers time points up $t_{dz}$, proofs can focus on $t \geq t_{dz}$. 

\noindent\textbf{Reacting to Dead Zones.}
The challenge introduced by the dead zones is to maintain collision-freedom even after $t_{dz}$, when some robots must alter their original trajectories to avoid or enter the dead zone.
When a dead zone is announced at time $t_{dz}$, the system must determine how each robot responds based on whether its remaining trajectory $\sigma^{\geq t_{dz}}(r_i)$ intersects $C_{dz}$. 
We define several cases that result in a fleet strategy for reacting to the DZ.
Unaffected robots, whose remaining trajectories do not intersect with $C_{dz}$ at all, continue on their original path. That is, for robots whose trajectories do not intersect with the dead zone, the DZ-handling trajectory coincides with the original planned trajectory, and the robot continues to follow it unchanged:\footnote{
We slightly abuse notations and write $\sigma^{\geq t_{dz}}(r_i) ~ \cap ~ C_{dz} \neq \emptyset$ for $\{ c \in \sigma^{\geq t_{dz}} ~\vert ~ c \in C_{dz} \}$}
\begin{equation}
\label{eq:unaffected_r_alt}
\begin{split}
\hat{\sigma}(r_i) ~=_{def}~ \sigma(r_i) \qquad \mbox{for } ~ r_i \in R ~ \mbox{ with }
~ \sigma^{\geq t_{dz}}(r_i) ~ \cap ~ C_{dz} = \emptyset.
\end{split}
\end{equation}

If a robot's remaining path is affected by the DZ, the simplest way for the system to react is to let the affected robot drive normally until it reaches the dead zone and stop there. With this, if a robot's remaining trajectory from time $t_{dz}$ intersects DZ, its goal is reassigned to be within it.
For robot $r_i \in R$ with $\sigma^{\geq t_{dz}}(r_i) ~ \cap ~ C_{dz} \neq \emptyset$ let $t'$ be the earliest time at or after $t_{dz}$ for which 
$\sigma_{t'} \in C_{dz}$. Then, we define $\hat{\sigma}$ as:
\begin{equation}
\label{eq:basic_dz_affected_enters_arena}
\hat{\sigma}_t(r_i) =_{def}
\begin{cases}
\sigma_t(r_i) & \mbox{ for } t < t' \\
\sigma_{t'}(r_i) & \mbox{ for } t \geq t' \\
\end{cases}
\end{equation}
This trivial strategy should already affect much fewer robots than a global emergency stop. In the next subsection, we turn to more advanced strategy for escaping or avoiding dead zones. We call the strategy that implements Eq.~\ref{eq:robot_in_dz_alt}, Eq.~\ref{eq:unaffected_r_alt}, and Eq.~\ref{eq:basic_dz_affected_enters_arena} \textbf{Basic DZ Handling}.
The safety of this simple strategy follows directly from these definitions: Eq.~\ref{eq:unaffected_r_alt} keeps original safe trajectories for robots that do not enter the dead zone, guaranteeing Eq.~\ref{eq:collision_free} for these robots and the definition in Eq.~\ref{eq:basic_dz_affected_enters_arena} makes robots stop upon entering the dead zone, guaranteeing Eq.~\ref{eq:robot_in_dz_alt} and Eq.~\ref{eq:collision_free} for the remaining robots.


\section{Escaping Dead Zones}
\label{sec:escaping}

Robots that are sufficiently far from a dead zone at the time of its announcement may execute an escape strategy that reroutes them around it. The decision depends on a safety threshold $\theta \in \mathbb{N}$: a robot must be at least $\theta$ cells away from $C_{dz}$ at $t_{dz}$ to have sufficient time and space to redirect before reaching the dead zone boundary. In practice, $\theta$ should additionally account for planning latency and communication delays; in our discrete unit-speed model both are assumed negligible, so $\theta = 1$ suffices. In a physical deployment, $\theta$ would need to be adjusted accordingly. 

In essence, many escape strategies can be conceived: a robot could return to its starting position and re-plan from there, wait at a safe cell, or bypass the dead zone entirely. In general, an escape strategy is a function that transforms a robot's remaining planned trajectory to avoid the dead zone. For the set $\mathcal{T}$ of all possible trajectories and all possible dead zones $\mathcal{DZ}$, let
$E_n: \mathcal{T} \times \mathcal{DZ} \to \mathcal{T}$
be a function that maps a planned trajectory $\sigma(r_i)$ and a dead zone $DZ$ to an escape trajectory $\hat{\sigma}(r_i)$ given $\sigma^{\geq t_{dz}}(r_i) \cap C_{dz} \neq \emptyset$, such that $\hat{\sigma}_t(r_i) = \sigma_t(r_i)$ for $t \leq t_{dz}$ and $\hat{\sigma}^{\geq t_{dz}}(r_i) \cap C_{dz} = \emptyset$. Of course, escape trajectories generated by $E_n$ should ensure the safety property stated in Eq.~\ref{eq:collision_free}.
The escape mechanisms, including the escape route definition, entry and exit points of a counterclockwise escape strategy, are presented in the following, before analyzing conditions for maintaining safety under the escape strategy.

\subsection{Counterclockwise Escape Trajectories}
In this paper we concentrate on bypassing strategies, which route robots around 
dead zones along a safe perimeter, analogous to how people naturally navigate around a hazardous area by following its boundary until they can continue towards their destination.
Among the many possible bypassing strategies, we define and analyze the counterclockwise (CCW) escape strategy $E_{ccw}$, which leads a robot around an obstacle by following its boundary in counterclockwise direction without stopping. We choose this direction arbitrarily, as it is simple to implement; our collision-freedom analysis for the counterclockwise strategy applies unchanged to a clockwise variant.
The perimeter $P_e$ around the dead zone is defined as:
\begin{equation}
\label{eq:escape_route_arena}
P_e = \{c_u \in C \setminus C_{dz}\ \mid \exists c_v \in C_{dz}: 
\vert c_u - c_v\vert _\infty = 1\ \}
\end{equation}
This represents the full bounding ring of cells surrounding the dead zone, corners included, but excluding the dead zone itself. On a grid, the perimeter is connected, i.e., a robot can move along
the perimeter.

For a robot $r_i$ with $\sigma^{\geq t_{dz}}(r_i) \cap C_{dz} \neq \emptyset$, we define the entry and exit points as the cells where the robot transitions onto and away from the escape route:

\begin{equation}
\label{eq:entry_exit_combined_arena}
\begin{split}
t_{in}^i = \min\{t > t_{dz} \mid \sigma_{t-1}(r_{i}) \notin P_e ~ \land ~ \sigma_{t}(r_i) \in P_e\} & \quad \mbox{and} \quad
c_{in}^i = \sigma_{t_{in}^i}(r_i) \\
t_{out}^i = \max\{t > t_{dz} \mid \sigma_t(r_i) \in P_e ~\land ~ \sigma_{t+1}(r_i) \notin P_e\} & \quad \mbox{and} \quad
c_{out}^i = \sigma_{t_{out}^i}(r_i)
\end{split}
\end{equation}
where $t_{in}^i$ is the time when robot $r_i$ enters the escape route, $c_{in}^i$ is the corresponding entry cell, $t_{out}^i$ is the time when the robot exits the escape route, and $c_{out}^i$ is the exit cell.
Let then $\sigma_e(r_i) = [c_{in}^i, \ldots, c_{out}^i]$ be the trajectory of $r_i$ that travels along the perimeter of the dead zone in CCW direction from $c_{in}^i$ to $c_{out}^i$, beginning at time $t_{in}^{i}$ at unit speed without waiting.

A robot can take a detour or shortcut by taking the escape route and will in general not arrive at $c_{out}^i$ at time $t_{out}^i$. We analyze how this
affects system safety in the next section. For now,
we focus on defining the possible delay (or speed-up) and describing the complete escaping strategy. 
Let $\vert \sigma_e(r_i)\vert  \leq \vert P_e\vert $ be the length of the 
segment on the perimeter. When a robot uses $\sigma_e(r_i)$ instead of its original route from $c_{in}^i$ to $c_{out}^i$, 
it will be delayed or sped up by 
$\delta_i = \vert \sigma_e(r_i)\vert  - (t_{out}^i - t_{in}^i) -1$.\footnote{We subtract 1 as $\sigma_e(r_i)$ contains all cells on the paths but $(t_{out}^i - t_{in}^{i})$ is the number of steps.}
 
\begin{equation}
\label{eq:ccw_escape_strategy_arena}
\hat{\sigma}_t(r_i) =_{def}
\begin{cases}
\sigma_t(r_i) & \mbox{ for } t < t_{in}^i \\
\sigma_{e,t}(r_i) & \mbox{ for } t_{in}^i \leq t \leq t_{out}^i + \delta_i \\
\sigma_{t-\delta_i}(r_i) & \mbox{ for } t > t_{out}^i + \delta_i
\end{cases}
\end{equation}
where we assume $\sigma_{t-\delta_i}(r_i) = \sigma_{\mu}(r_i)$ for $t-\delta_i > \mu$. 

\noindent\textbf{The Complete Strategy.} 
We apply escaping only conditionally as (1) we assume that some time is needed to execute the change of plan and (2) escaping creates deviations from originally planned positions over time that may lead to collisions. We introduce a maximum delay bound $\delta_{max}$ and prove collision freedom under some constraints for a concrete $\delta_{max}$ in the next subsection.
Formally, the DZ-handling trajectory $\hat{\sigma}(r_i)$ for robots whose remaining trajectory intersects $C_{dz}$ is constructed depending on the robot's distance from the dead zone at $t_{dz}$. If $dist(\alpha_{t_{dz}}(r_i), C_{dz}) \geq \theta$, the robot has sufficient time and space to reroute, and $\hat{\sigma}(r_i)$ is constructed tentatively as the CCW escape trajectory in Eq.~\ref{eq:ccw_escape_strategy_arena}. 
If $\vert \delta_i\vert  \leq \delta_{max}$, then $\hat{\sigma}(r_i)$ is enacted.
Otherwise, if
$dist(\alpha_{t_{dz}}(r_i), C_{dz}) < \theta$
or if $\vert \delta_i\vert  > \delta_{max}$, the robot has insufficient time to reroute safely or might violate the distance with other robots and the system falls back to $\hat{\sigma}(r_i)$ as constructed in Eq.~\ref{eq:basic_dz_affected_enters_arena} --- the robot follows its original trajectory until it enters $C_{dz}$ and stops there.

The two handling strategies, \emph{Basic DZ Handling} and \emph{DZ Handling with Escaping}, define how the fleet responds to a Dead Zone announcement. While Basic DZ Handling is trivially safe by construction, the collision-freedom of the escaping strategy is not obvious and requires formal analysis.  

\begin{remark}[Computational Cost per Dead Zone Event]
Let $k = \vert R\vert$ be the fleet size and $L$ the length of the longest remaining trajectory at $t_{dz}$. The system's reaction to a dead zone is a single loop over the fleet: for each robot $r_i$, check $\sigma^{\geq t_{dz}}(r_i)$ for intersection with $C_{dz}$ ($O(L)$); if there are none, $r_i$ is unaffected (Eq.~\ref{eq:unaffected_r_alt}). If an intersection is found, $r_i$ is classified as dead (Eq.~\ref{eq:basic_dz_affected_enters_arena}) or escaping in $O(1)$, using the Complete Strategy above, evaluated on data already produced by the check. A dead robot's trajectory is truncated in $O(1)$. An escaping robot's route is found by locating its entry and exit points (Eq.~\ref{eq:entry_exit_combined_arena}) and walking $P_e$ once, bounded by $\vert \sigma_e(r_i)\vert  \leq \vert P_e\vert $, costing $O(\vert P_e\vert )$ (Eq.~\ref{eq:ccw_escape_strategy_arena}). Each robot therefore costs $O(L + \vert P_e\vert )$, giving a total reaction cost of $O\big(k \cdot (L + \vert P_e\vert )\big)$. In practice, this bound is conservative, as typically only a few robots in the fleet are affected by the dead zone.
\end{remark}

\subsection{Collision-Freedom Guarantees for Escaping}
\label{sec:problem_definition}

We now establish collision-freedom guarantees for the presented escaping strategy. We analyze one specific setup: counterclockwise unidirectional movement with no waiting on the escape route.
The question we aim to answer is: can we impose constraints on the initial planning of $\sigma$ that ensures collision freedom under the presented escaping strategy, making it \textit{safe by design}. 
While there could be many such constraints, we focus on the minimum distance that has to hold at all times between robots in the initial planning in order for escape maneuvers not leading to collisions.
Intuitively, if two robots are far enough apart, a (short) detour or shortcut will not suffice to cover that distance. We formalize this argument in the remainder of this section.

We start by establishing an upper bound on the maximum spatial deviation that a robot $r_i$ can experience under the escaping strategy's $\hat{\sigma}(r_i)$, i.e., a bound on $dist(\hat{\sigma}_t(r_i), \sigma_t(r_i))$.

\begin{lemma}[Maximum Deviation Bound]
\label{lem:max_deviation}
Any robot that is not in the dead zone deviates from its original planned position by at most $max(\delta_{max},|P_e|)$:
\begin{equation*}
\forall r_i \in R \, \forall t \in [0, \mu]: ~ 
\hat{\sigma}_t(r_i) \notin C_{dz} ~ \Rightarrow ~
dist(\hat{\sigma}_t(r_i), \sigma_t(r_i)) ~ \leq ~ max(\delta_{max},|P_e|). 
\end{equation*}
\end{lemma}

\begin{proof}
For all robots that are not affected by DZ, we have $\hat{\sigma}(r_i) = \sigma(r_i)$. For robots that enter the dead zone $dist(\hat{\sigma}_t(r_i), \sigma_t(r_i)) = 0$ for $t \leq t_{in}^i$ and then $\hat{\sigma}_t(r_i) \in C_{dz}$ for $t>t_{in}^i$.
Then, for an escaping robot $r_i$, its trajectory is defined in Eq.~\ref{eq:ccw_escape_strategy_arena} as $\hat{\sigma}_t(r_i) = \sigma_t(r_i)$ for $t < t_{in}^i$ and as $\hat{\sigma}_t(r_i) = \sigma_{t^+\delta_i}(r_i)$ for $t > t_{out}^i + \delta_i$, guaranteeing that $dist(\hat{\sigma}_t(r_i), \sigma_t(r_i)) \leq \delta_{max}$ for these times.
For $t_{in}^i \leq t \leq t_{out}^i + \delta_i$, the robot $r_i$ travels along the perimeter of the dead zone that is a closed loop of length $|P_e|$. We have to analyze by how much $\hat{\sigma}(r_i)$ can deviate from $\sigma(r_i)$ in this time. After $m\leq |P_e|$ steps, the escaping route must be able to return to $c_{in}^{i}$ in $|P_e|-m$ steps, so it can be at most $|P_e|-m$ away from $c_{in}^{i}$. In the worst case, the original trajectory can be at most $m$ away from $c_{in}^{i}$. By the triangle inequality,
$dist(\hat{\sigma}_t(r_i), \sigma_t(r_i)) \leq (|P_e|-m)+m = |P_e|$ for $t_{in}^{i} \leq t \leq t_{out}^i + \delta_i \leq t_{in}^{i} + |P_e|$.
Together, these cases establish the lemma. 
\end{proof}

The maximum delay $\delta_{max}$ impacts the maximum deviation directly. We introduced $\delta_{max}$ as without further assumptions on the initial path planning a robot's path theoretically may cross a dead zone multiple times and take long detours in the meantime. Path planning algorithms such as CBS~\cite{sharon2015cbs} or prioritized A*~\cite{silver2005cooperative} aim to produce shortest trajectories but it is not guaranteed in general and may force robots into longer detours that loop around other robots. Under an unfortunately placed dead zone such a detour may cross $C_{dz}$ more than once and be shortened considerably during escape. Such cases are rare outliers. Most of the time, we expect the escape route to be a short detour from a path through the dead zone. Hence, setting $\delta_{max} \leq |P_e|$ does not affect many robots. 

The aim of the dead zone is to fence off the anomaly. While we do not discuss how its shape is determined, we suggest convex forms and use rectangles for simplicity. Since our analysis is based on the perimeter of the dead zone, it must be known in advance. Therefore, we fix a maximum dead zone size $W \times H$ as a system parameter. This fixes the maximum perimeter $|P_e|$, and collision-freedom is guaranteed for all dead zones up to this size.
We can now establish the following safety guarantee under these assumptions: a single rectangular dead zone of at most the maximum size, announced at one time point; unit-speed robots with no waiting on the CCW escape route; no online replanning during escape; and a dead zone sufficiently far from the grid boundary so that its perimeter $P_e$ is fully traversable (no edge cases).

\begin{theorem}[Collision-Free Escape with Bounded Delay]
\label{thm:collision_free_escape}
If $\delta_{max} \leq |P_e|$ and robots on their original collision-free trajectories maintain minimum pairwise separation $d_{safe} = 2|P_e| + 1$ at any point of time, then no collisions occur during or after escape around a dead zone announced at time point $t_{dz}$.
\end{theorem}

\begin{proof}
\textbf{Preliminary case.} At any $t \geq t_{dz}$, suppose at most one of $r_i, r_j$ deviates from its original trajectory --- i.e. one robot is escaping while the other has not yet started or does not need to escape, so its deviation from the original trajectory is $0$. By the triangle inequality:$dist(\sigma_t(r_i),\sigma_t(r_j)) \leq dist(\alpha_t(r_i),\alpha_t(r_j)) + |P_e| + 0$. Rearranging, and using $dist(\sigma_t(r_i),\sigma_t(r_j)) \geq d_{safe} = 2|P_e|+1$:
\begin{equation}
dist(\alpha_t(r_i),\alpha_t(r_j)) \geq (2|P_e|+1) - |P_e| = |P_e|+1 > 0
\end{equation}
Since $|P_e|+1 > 1$, this rules out both a cell collision ($dist=0$) and a swap collision (which requires the robots to be adjacent, $dist=1$, at consecutive timesteps). With that, the worst case can only occur, when both robots deviate from their original trajectory, i.e. both escaping.

\textbf{Worst case.} Suppose a collision occurs between two robots $r_i$ and $r_j$, meaning they either occupy the same cell at some time $t \in [t_{dz}, \mu]$, i.e., $\alpha_t(r_i) = \alpha_t(r_j)$, or swap cells between two consecutive time points $t$ and $t+1$, i.e., $\alpha_t(r_i) = \alpha_{t+1}(r_j)$ and $\alpha_{t+1}(r_i) = \alpha_t(r_j)$. We consider both cases in turn.

\textit{Cell collision.} Suppose $\alpha_t(r_i) = \alpha_t(r_j)$. Since the CCW escape strategy deviates each robot from its original trajectory by at most $|P_e|$ as a consequence of Lemma~\ref{lem:max_deviation} and $\delta_{max} \leq |P_e|$, the escaping robot's position at time $t$ differs from its originally planned position by at most $|P_e|$. Therefore, by the triangle inequality:
\begin{equation}
dist(\sigma_t(r_i), \sigma_t(r_j)) \leq 
dist(\alpha_t(r_i), \alpha_t(r_j)) + 2|\sigma_e| = 
0 + 2|P_e| = 2|P_e|
\end{equation}
This means the original trajectories had pairwise separation strictly less than $2|\sigma_e| + 1 = d_{safe}$ at time $t$, contradicting our assumption. Hence no cell collision can occur.

\textit{Swap collision.} The original planned trajectories are swap-free by Eq.~\ref{eq:collision_free_assume}. We show that the escape maneuver cannot introduce a swap.
Suppose $r_i$ and $r_j$ would swap at cells $c_a$ and $c_b$, with $r_i$ visiting $c_a \to c_b$ and $r_j$ visiting $c_b \to c_a$. On the original trajectories, let $r_i$ be the first to visit these cells, arriving at $c_a$ at time $t_1$ and $c_b$ at $t_1+1$, and $r_j$ arriving at $c_b$ at time $t_2$ with $t_2 \geq t_1 + 2|P_e|+1$ by the $d_{safe}$ constraint. The worst case occurs when $r_i$ delays and $r_j$ accelerates, as the reverse combination only widens the time gap between them. 
After escape, $r_i$ delays by at most $|P_e|$, so it visits $c_b$ at most at time $t_1 + 1 + |P_e|$. Meanwhile, $r_j$ accelerates by at most $|P_e|$, so it visits $c_b$ at earliest at time $t_2 - |P_e| \geq t_1 + |P_e| + 1$. For a swap to occur, $r_j$ must be at $c_b$ while $r_i$ is still at $c_a$, i.e.\ at time $t_1 + |P_e|$ at the latest. But $r_j$ cannot reach $c_b$ before $t_1 + |P_e| + 1$; it is always at least one time step too late. Hence no swap collision can occur under $d_{safe} = 2|P_e|+1$.
\end{proof}

This means that we can establish a collision freedom
guarantee if we know a bound on the length of the
perimeter of an occurring dead zone and can assume that
only one dead zone is announced in the considered time span. 
Since $d_{safe}$ grows directly with $|P_e|$, larger dead zones impose increasingly severe separation constraints and quickly limit fleet throughput, as we analyze in Section~\ref{sec:evaluation}. 
Better guarantees may be obtained by considering potential dead zones during initial planning.

Turning from safety to performance, the following corollary follows from the lemma and theorem immediately.

\begin{corollary}[Bounded Delay]
\label{thm:collision_free_delay}
For $\delta_{max} \leq |P_e|$, every escaping robot $r_i$ arrives at its destination with delay $|\delta_i| \leq |P_e|$.
\end{corollary}

While individual robots are not delayed significantly by small dead zones, we expect the safety envelope $d_{safe}$ to limit how many robots can operate effectively on a grid.

\section{Evaluation}
\label{sec:evaluation}

We evaluate the different strategies for handling dead zones in a series of simulated experiments.
We use a grid with $50 \times 50$ cells in which border cells (1 cell wide) are reserved exclusively as robot start/goal positions and may not be traversed. On this grid we operate a varying number of robots that each perform a sequence of transport jobs. 
The inner grid ($48 \times 48$) is used for path planning and driving, where we use CBS~\cite{sharon2015cbs} for the initial schedule and prioritized planning with space-time A*~\cite{silver2005cooperative} for subsequent jobs to generate collision-free schedules.
The planner operates on the inner grid only, mapping border start (e.g. $(x,0)$) to inner start (e.g. $(x,1)$). Robots enter grid continuously from border cells; a robot finishing a job immediately receives a new randomly assigned border-cell goal ($\geq 50$ Manhattan distance away, currently not assigned as a goal to another robot). In every simulation, we simulate $T = 1{,}000$ time steps, allowing for between 14 and 15 sequential jobs for a robot.

At a random point in time, a dead zone is announced. We vary sizes, randomize location and limit position with a margin: 2 cells from each border, ensuring that escape 
routes $P_e$ lie entirely on inner cells.

While the CCW escape strategy is proven collision-free under $d_{safe}$, this separation constraint directly limits how many robots can operate simultaneously, affecting system throughput. To investigate this tradeoff, we evaluate two additional variants of the escape strategy with relaxed separation requirements, alongside the proven-safe strategy, basic DZ handling, and a full emergency stop. We compare the following dead zone handling and escaping strategies:
\begin{itemize}
    \item \textit{Emergency Stop (ES):} all robots stop in place at $t_{dz}$; jobs thereafter are lost.
    \item \textit{Basic DZ (BDZ):} robots whose remaining trajectory intersects $C_{dz}$ enter the DZ and stop; unaffected robots continue; new trajectories are being planned around the DZ as a static obstacle.
    \item \textit{EDZ$_{safe}$ ($d_{safe} = 2\vert P_e\vert +1$):} robots whose remaining trajectory intersects $C_{dz}$ execute the CCW escape; collision-freedom is guaranteed by Theorem~\ref{thm:collision_free_escape}.
    \item \textit{EDZ$_1$ ($d = \vert P_e\vert +1$):} same CCW escape logic with a relaxed safety margin. This variant is not covered by Theorem~\ref{thm:collision_free_escape} and its collision-freedom is not guaranteed.
    \item \textit{EDZ$_{0.5}$ ($d = 0.5\,\vert P_e\vert $):} further relaxed margin; collision-freedom is no longer guaranteed by Theorem~\ref{thm:collision_free_escape}.
    \item \textit{Baseline:} an identical system with no DZ or emergency stop event, to quantify the throughput impact of each strategy.
\end{itemize}
Together, $\text{EDZ}_1$ and $\text{EDZ}_{0.5}$ allow us to assess how sensitive throughput is to the choice of separation, even though their collision-freedom is not guaranteed.

In our implementation, a robot executes the escape maneuver if its remaining trajectory intersects $C_{dz}$ and it is not already inside it at $t_{dz}$, corresponding to $\theta = 1$ as stated above. This choice is natural in the unit-speed discrete model, where a robot one step outside $C_{dz}$ always has sufficient time to redirect before entering the DZ. Additionally, our implementation considers another rare outlier: a situation where a robot's remaining trajectory does not intersect the $C_{dz}$, but has shared cells with $\sigma_e$. Although such a robot does not need to escape, it uses a shared resource. In this case we force this robot to escape as well.

\subsection{Experimental Setup}
A \textit{configuration} is one full simulation run defined by: (i) DZ size; (ii) maximum robot count $N$; (iii) the randomly generated job sequence; and (iv) one random DZ event $(t_{dz}, x_{dz}, y_{dz})$. The DZ announcement time $t_{dz}$ is sampled uniformly from $[T/6,\, T/3] = [166,\, 333]$ to allow sufficient warmup, and the DZ position is sampled uniformly from all valid inner positions such that at least one active robot's remaining trajectory intersects $C_{dz}$. We also set $\delta_{max} = \vert P_e\vert $ throughout, the least restrictive value permitted by Theorem~\ref{thm:collision_free_escape}.
When a dead zone is announced at $t_{dz}$, the fleet $R$ is partitioned into two mutually exclusive sets: operating robots $R_{op} = \{r_i \in R \mid \alpha_{t_{dz}}(r_i) \notin C_{dz}\}$ and dead robots $R_{dz} = \{r_i \in R \mid \alpha_{t_{dz}}(r_i) \in C_{dz}\}$. Dead robots receive no further job assignments; operating robots continue normally.
After $t_{dz}$, newly assigned jobs plan around the DZ as a permanent obstacle. 
The following variables are swept independently:
\begin{itemize}
    \item Maximum number of robots: $N \in \{5, 10, 15, 20, 25\}$
    \item Maximum DZ size: $1{\times}1$, $1{\times}2$, $2{\times}2$, $2{\times}3$, $3{\times}3$
    \item Approach: ES, BDZ, EDZ$_{safe}$, EDZ$_1$, EDZ$_{0.5}$, Baseline   
\end{itemize}
Each combination of approach, $N$, and DZ size is evaluated over 30 independent seeds. Results are reported as mean $\pm$ standard deviation. Experiments were also conducted on larger grids ($50{\times}100$, $100{\times}100$) with up to $N = 100$ robots and DZ sizes up to $5{\times}5$. As no qualitatively new trends emerged, we report results for the $50{\times}50$ grid only.

\begin{figure}[t!]
    \centering
    \begin{subfigure}[t]{0.48\textwidth}
        \includegraphics[width=\textwidth]{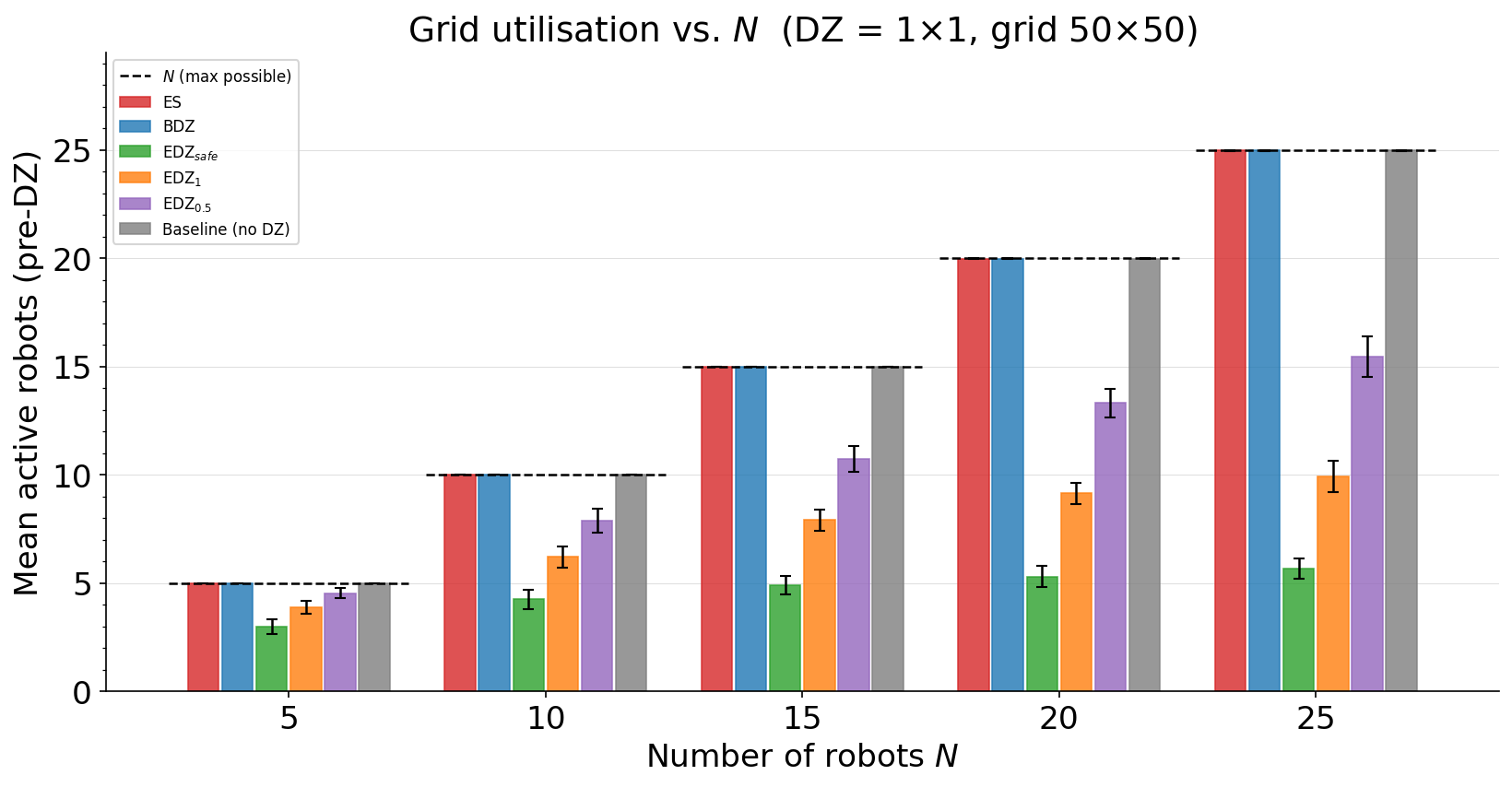}
        \caption{Mean active robots before $t_{dz}$. ES and BDZ utilize the full fleet; EDZ$_{safe}$ saturates at $\approx 5$-$6$ robots.
        Relaxed variants EDZ$_1$ and EDZ$_{0.5}$ achieve progressively higher utilization at the cost safety guarantees. Dashed lines mark fleet size $N$.}
        \label{fig:util_vs_N}
    \end{subfigure}
    \hfill
    \begin{subfigure}[t]{0.48\textwidth}
        \includegraphics[width=\textwidth]{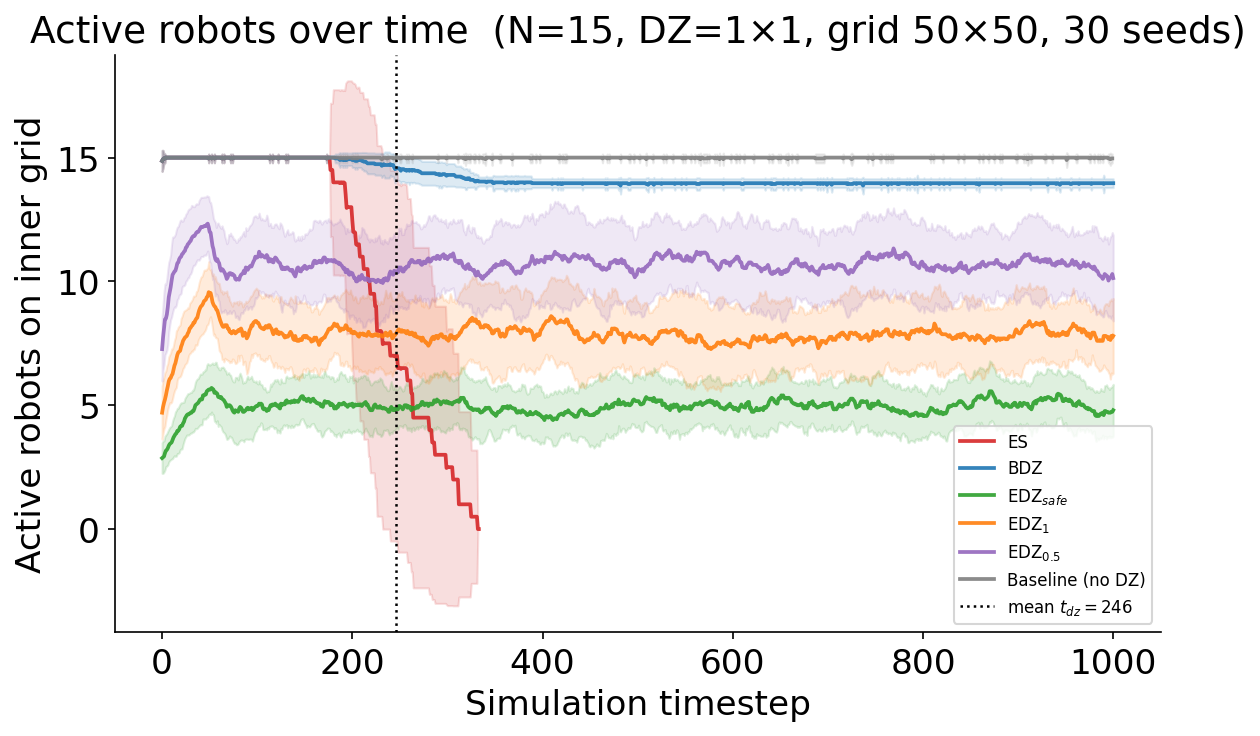}
        \caption{Active robots on the inner grid over time ($N=15$). ES drops to zero at $t_{dz}$; BDZ loses the robots whose paths crossed DZ; all EDZ variants maintain pre-DZ throughput with no drop at $t_{dz}$. EDZ throughput decreases with stricter distance $d$ or $d_{safe}$. Dotted line marks mean $t_{dz} = 246$.}
        \label{fig:active_robots_over_time}
    \end{subfigure}

    \vspace{0.5em}

    \begin{subfigure}[t]{0.48\textwidth}
        \includegraphics[width=\textwidth]{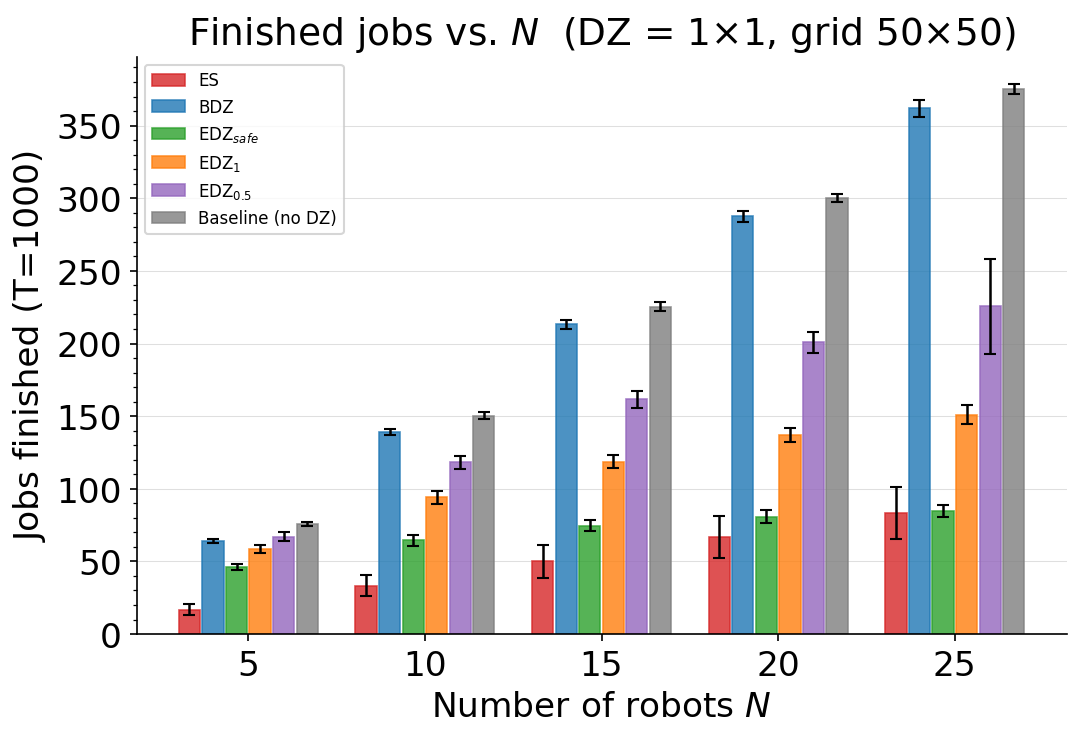}
        \caption{Total finished jobs over $T=1000$ steps vs. fleet size $N$. BDZ achieves near-baseline throughput across all $N$. EDZ$_{safe}$ saturates at $\approx 75$--$85$ jobs for $N \geq 15$; relaxed variants EDZ$_1$ and EDZ$_{0.5}$ recover a significant portion of the throughput gap.}
        \label{fig:jobs_vs_N}
    \end{subfigure}
    \hfill
    \begin{subfigure}[t]{0.48\textwidth}
        \includegraphics[width=\textwidth]{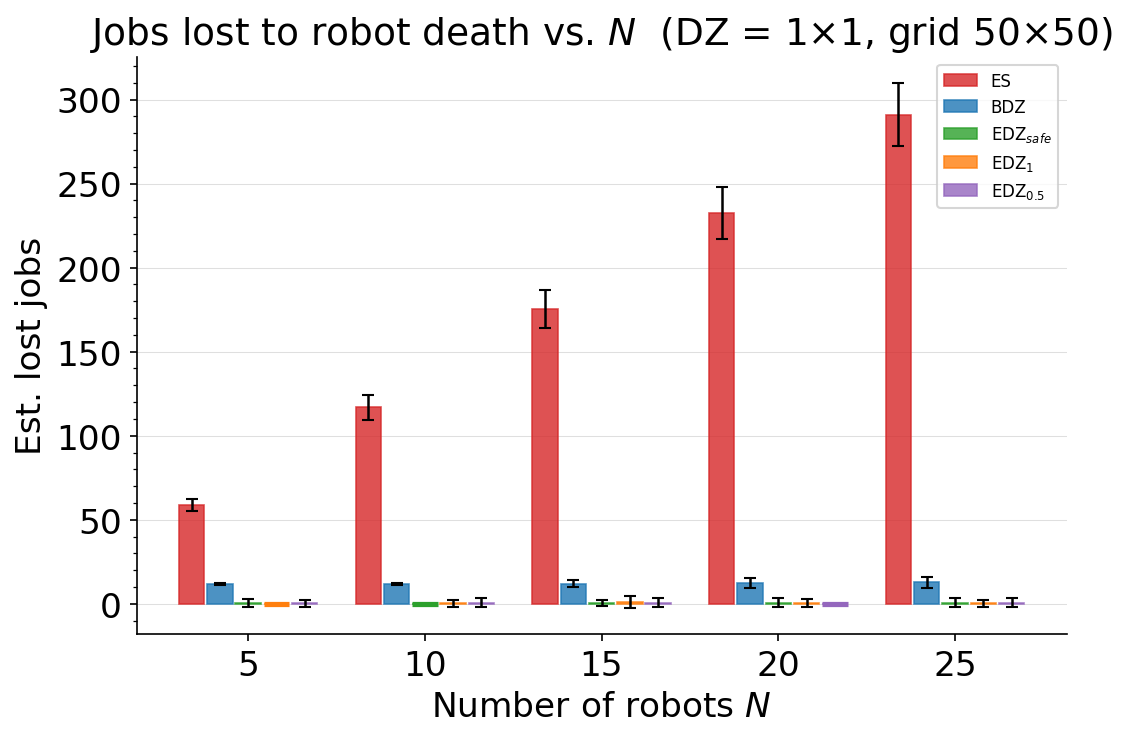}
        \caption{Mean jobs lost due to robot death vs.\ fleet size $N$. ES loses all post-$t_{dz}$ jobs, BDZ loses $\approx 12$--$16$ jobs regardless of $N$ (typically one robot is captured per run). All EDZ variants lose near-zero jobs across all $N$.
        }
        \label{fig:lost_vs_N}
    \end{subfigure}
    
    \caption{Simulation results: grid $50{\times}50$, $T=1000$, $30$ seeds per configuration. Fleet size sweep (DZ $1{\times}1$, fixed). All error bars show mean $\pm$ std.}
    \label{fig:max_number_robots}
\end{figure}

\begin{figure}[t!]
    \centering
    \begin{subfigure}[t]{0.48\textwidth}
        \includegraphics[width=\textwidth]{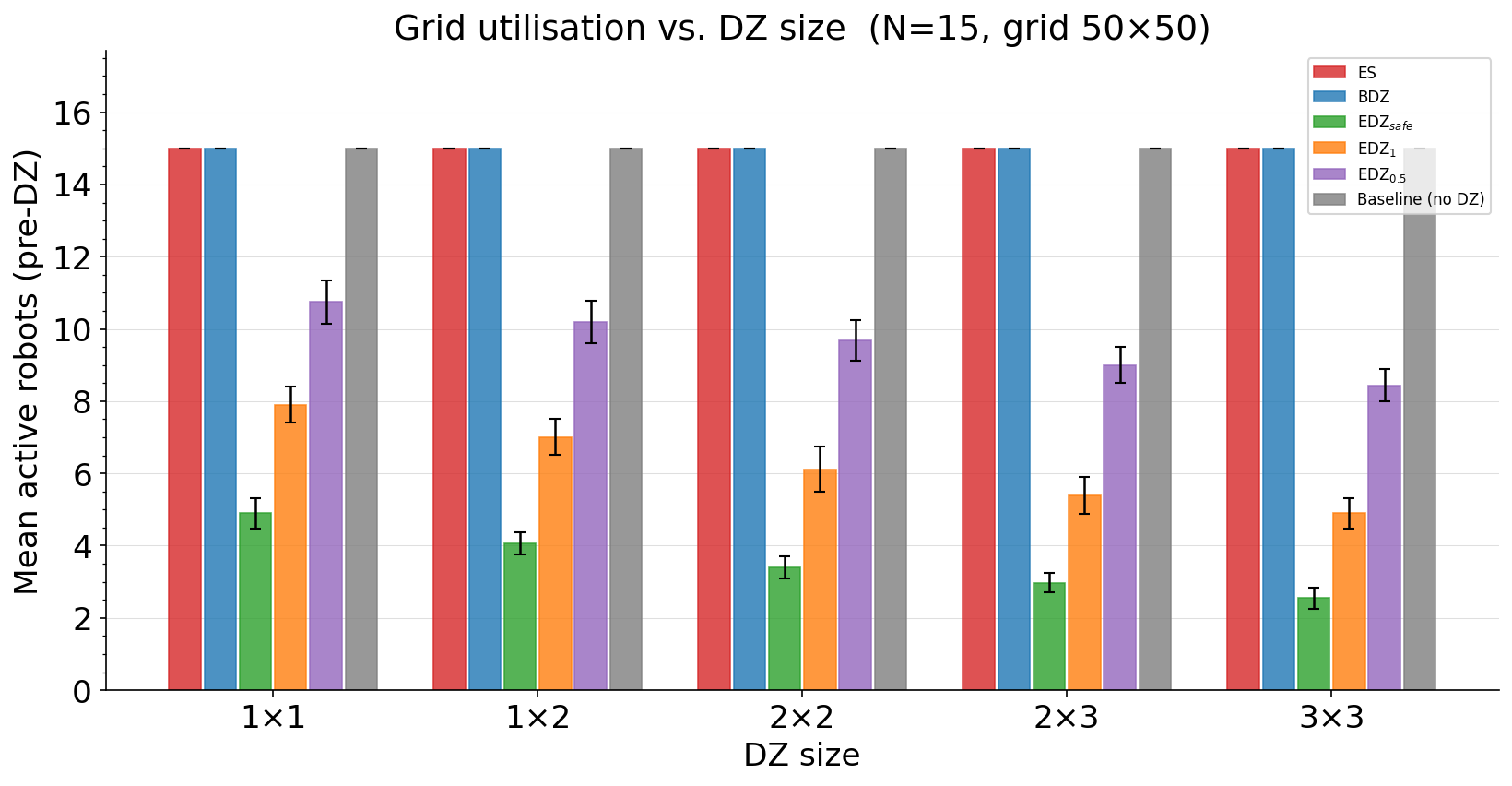}
        \caption{Mean active robots before $t_{dz}$ vs.\ DZ size. ES and BDZ maintain full fleet utilization ($\bar{a}=15$) independent of DZ size. All EDZ variants decrease monotonically with DZ size as $d_{safe}$ or $d$ grow; EDZ$_{safe}$ drops from $\approx 5$ to $\approx 2.5$, while EDZ$_{0.5}$ retains $\approx 8.5$ robots even at the largest DZ.}
        \label{fig:util_vs_dz}
    \end{subfigure}
    \hfill
    \begin{subfigure}[t]{0.48\textwidth}
        \includegraphics[width=\textwidth]{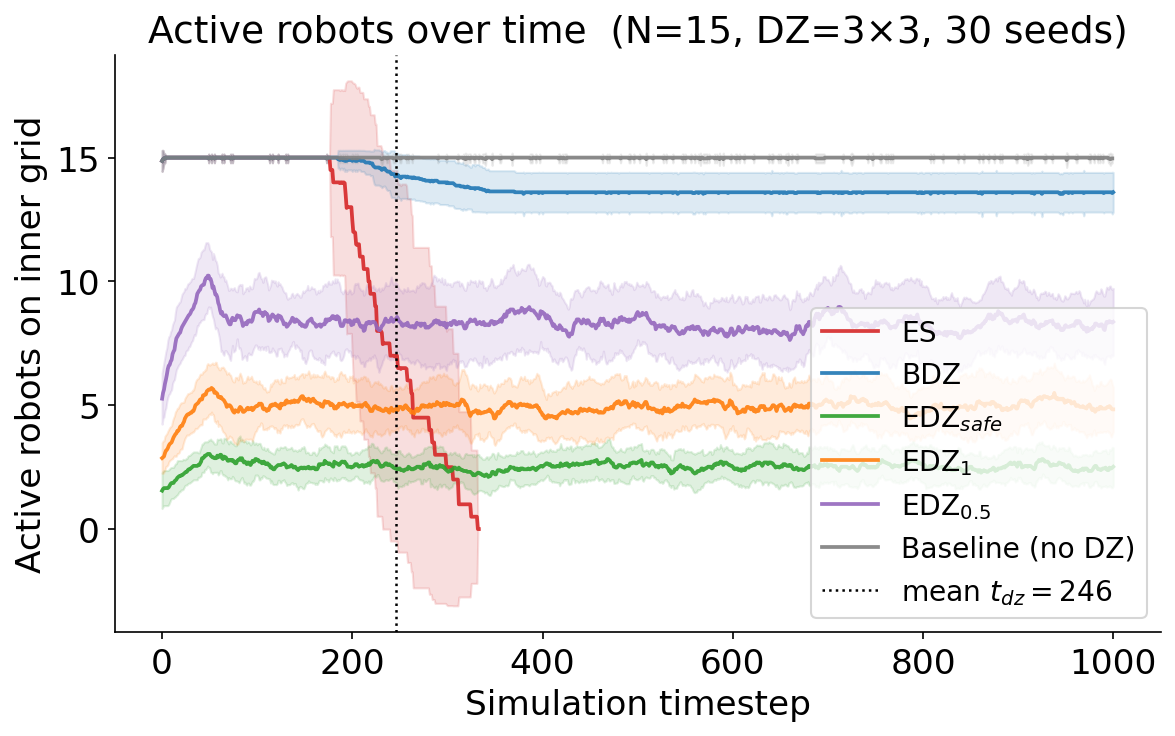}
        \caption{Active robots over time for DZ $3{\times}3$ (for comparison see Fig.~\ref{fig:active_robots_over_time}). ES drops to zero at $t_{dz}$; BDZ drops by $\approx 1$ robot and stabilizes. The EDZ ceiling drops visibly between the two panels for all variants, while BDZ and ES are unaffected by DZ size.}
        \label{fig:active_robots_over_time_large_dz}
    \end{subfigure}

    \vspace{0.5em}

    \begin{subfigure}[t]{0.48\textwidth}
        \includegraphics[width=\textwidth]{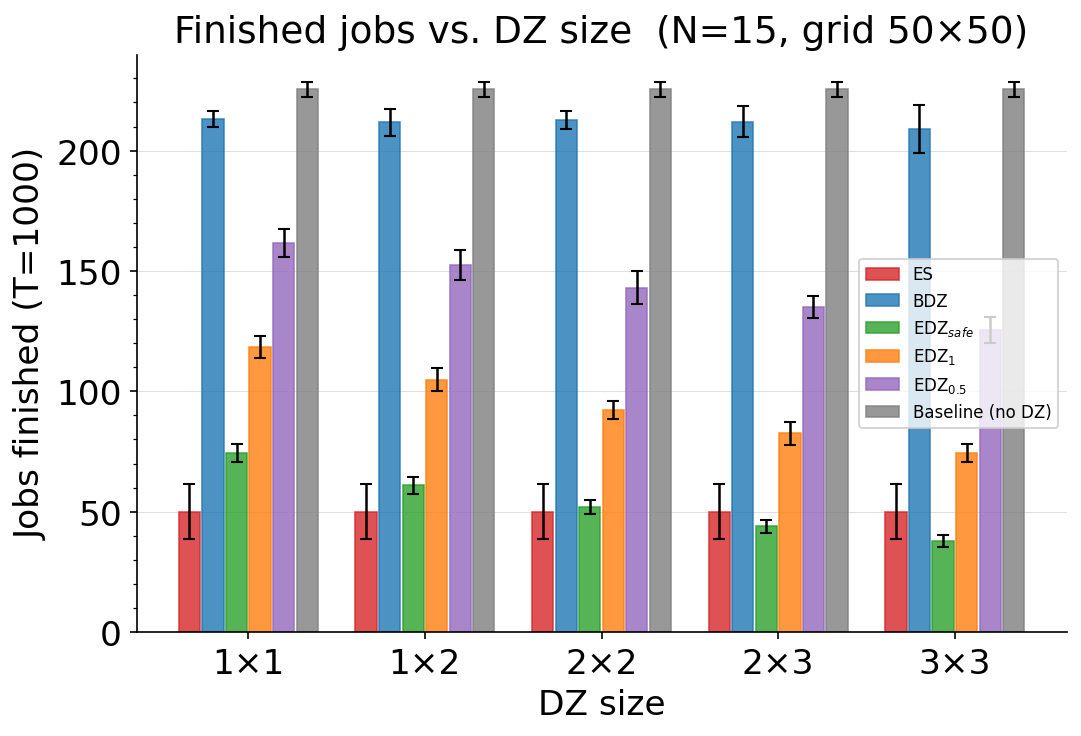}
        \caption{Total finished jobs vs. DZ size. ES is flat ($\approx 50$ jobs) and BDZ near-constant ($\approx 210$ jobs), both independent of DZ geometry. EDZ$_{safe}$ decreases from $74$ to $38$ jobs; EDZ$_1$ from $120$ to $75$; EDZ$_{0.5}$ from $160$ to $125$, all driven by the shrinking active fleet as $d_{safe}$ or $d$ grow.}
        \label{fig:jobs_vs_dz}
    \end{subfigure}
    \hfill
    \begin{subfigure}[t]{0.48\textwidth}
        \includegraphics[width=\textwidth]{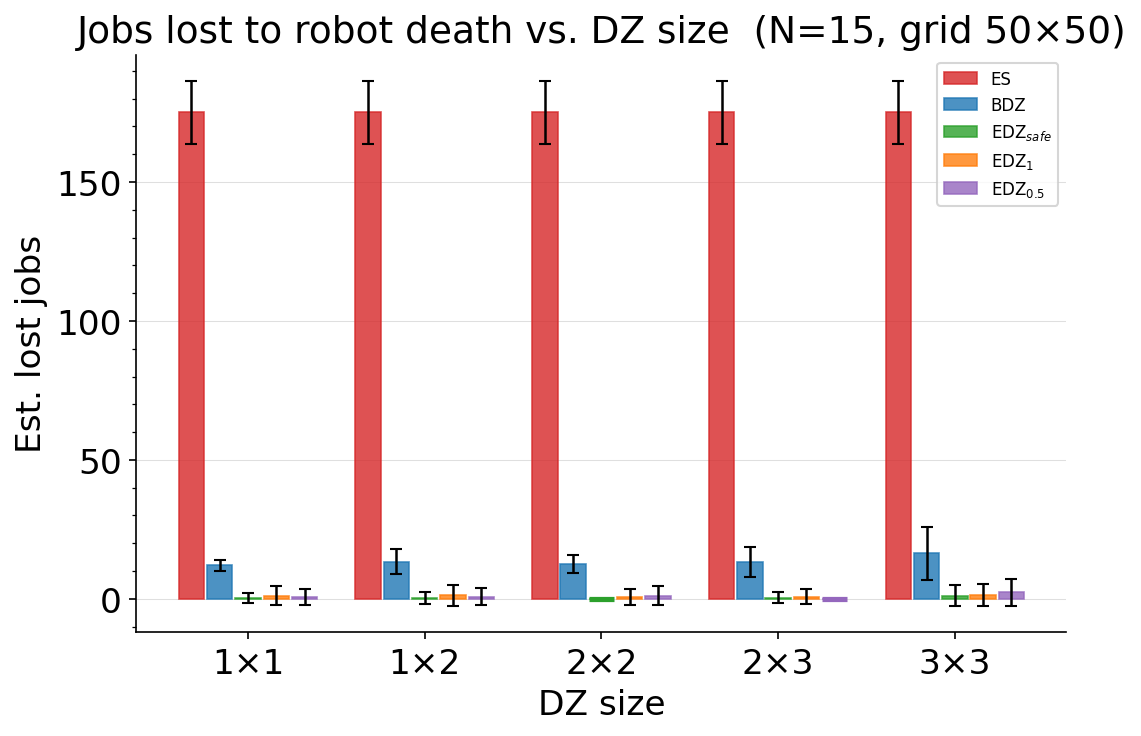}
        \caption{Mean jobs lost due to robot death vs. DZ size. ES loses $\approx 175$ jobs independent of DZ geometry. BDZ grows marginally from $\approx 10$ to $\approx 16$ jobs. All EDZ variants remain near zero across all DZ sizes, confirming robustness of the escape strategy independent of margin choice.}
        \label{fig:lost_vs_dz}
    \end{subfigure}

    \caption{Simulation results: grid $50{\times}50$, $T=1000$, $30$ seeds per configuration. DZ size sweep ($N=15$, fixed).
    All error bars show mean $\pm$ std.}
    \label{fig:DZ_size}
\end{figure}

\subsection{Experimental Results}
We collect a comprehensive set of metrics per run, including collision rate, mean active robots, finished jobs, lost jobs, robots dead in DZ, DZ exposure count, and delay bound validation. In this paper we focus on the metrics that best characterize the tradeoff between safety and throughput. \textbf{Mean active robots} ($\bar{a}$) measures fleet utilization before and after $t_{dz}$. \textbf{Finished jobs} counts total completed jobs over $T = 1{,}000$ steps. \textbf{Lost jobs} counts jobs not completed due to robot death in $C_{dz}$ (or because of emergency stop). \textbf{Collision rate} reports the fraction of runs with at least one robot-robot collision, serving as a sanity check for proven-safe approaches and a violation indicator for relaxed variants. 
For each run we also verified whether any robot's escape deviation exceeded $\delta_{max}$, which would indicate the need to escape DZ more than once; such robots are instead treated as dead in DZ by the fallback rule (Section~\ref{sec:escaping}). In no run did any robot require more than one escape of the same dead zone, showing that the spatio-temporal detours imposed by $d_{safe}$ did not produce trajectories crossing $C_{dz}$ multiple times under our experimental configurations.
As an additional sanity check, we verify that the delay bound $\delta_i \leq \vert P_e \vert$ holds for all escaping robots in EDZ$_{safe}$ runs, confirming Corollary~\ref{thm:collision_free_delay}.

We present results for two sweeps: fleet size scaling with fixed DZ size, and DZ size scaling with fixed fleet size. Both sweeps use the $50{\times}50$ grid with $T = 1{,}000$ time steps and 30 independent seeds per configuration.

\vspace{8px}
\noindent\textbf{Scaling with Fleet Size.}
We fix DZ size at $1{\times}1$ ($\vert P_e\vert  = 8$, $d_{safe} = 17$) and sweep $N \in \{5, 10, 15, 20, 25\}$ over $30$ independent seeds each. The delay bound was satisfied in $100\%$ of EDZ$_{safe}$ runs across all configurations. Qualitative trends are consistent across all 30 seeds.
Fig.~\ref{fig:max_number_robots} shows the results of the sweep. Before the DZ announcement, ES and BDZ scale linearly with $N$, utilizing the full fleet, while EDZ$_{safe}$ saturates early due to the $d_{safe}=17$ spacing constraint limiting simultaneous occupancy of the inner grid regardless of how many robots are queued at the border; the relaxed variants EDZ$_1$ and EDZ$_{0.5}$ progressively recover utilization (Fig.~\ref{fig:util_vs_N}). After the dead zone is announced, the strategies diverge sharply (Fig.~\ref{fig:active_robots_over_time}): ES loses all jobs dispatched after $t_{dz}$, while BDZ continues with a marginally reduced fleet near baseline throughput. All EDZ variants maintain their pre-DZ ceiling with no drop at $t_{dz}$ and lose near-zero jobs across all $N$ (Figs.~\ref{fig:jobs_vs_N},~\ref{fig:lost_vs_N}), confirming that the escape strategy preserves all but the robot already inside $C_{dz}$ at $t_{dz}$; EDZ$_1$ and EDZ$_{0.5}$ partially close the throughput gap to BDZ. Notably, EDZ$_{0.5}$ registers a collision rate of $3\%$ at $N=25$ with DZ size $1\times1$, confirming that the relaxed margin violates the conditions of Theorem~\ref{thm:collision_free_escape} and collisions do occur in practice, while EDZ$_{safe}$ and EDZ$_1$ maintain $0\%$ collision rate across all configurations.

\vspace{8px}
\noindent\textbf{Scaling with DZ Size.}
We fix $N = 15$ and sweep DZ sizes $1{\times}1$, $1{\times}2$, $2{\times}2$, $2{\times}3$, $3{\times}3$, with $d_{safe}$ scaling accordingly from $17$ to $33$ for EDZ$_{safe}$.
Fig.~\ref{fig:DZ_size} shows the results. ES and BDZ are entirely unaffected by DZ size in terms of fleet utilization (Fig.~\ref{fig:util_vs_dz}), while all EDZ variants decrease monotonically as the DZ grows, reflecting the increasing $d_{safe}$ or $d$ requirement (compare Fig.~\ref{fig:active_robots_over_time} and Fig.~\ref{fig:active_robots_over_time_large_dz}). This independence from dead zone geometry also holds for finished and lost jobs in ES and BDZ (Figs.~\ref{fig:jobs_vs_dz},~\ref{fig:lost_vs_dz}), while for EDZ variants the growing distance constraint $d$ or $d_{safe}$ directly shrinks the active fleet and with it the number of finished jobs. Despite this throughput reduction, all EDZ variants remain near zero in robot losses across all DZ sizes, confirming the robustness of the escape strategy regardless of the choice of margin. 

\subsection{Discussion}
Across all configurations, EDZ$_{safe}$ maintained a $0\%$ collision rate, empirically confirming Theorem~\ref{thm:collision_free_escape}. The results also reveal a clear tradeoff between safety guarantees and throughput. ES is the worst performer in all metrics, confirming that a full fleet halt is unnecessarily disruptive for a localized anomaly. BDZ achieves near-baseline throughput, losing on average $\approx 1-2$ robots per DZ event regardless of fleet size or DZ geometry, but permanently sacrifices those robots. All EDZ variants lose at most $1$ robot per run (the robot inside $\mathcal{C}_{dz}$ at $t_{dz}$), yet maintain their pre-DZ active ceiling (Fig.~2c); the border waiting pool dispatches a replacement within one cycle, leaving the throughput cost driven entirely by $d_{safe}$.

The relaxed variants EDZ$_1$ and EDZ$_{0.5}$ demonstrate that the throughput gap between BDZ and proven-safe EDZ$_{safe}$ is largely an artifact of the conservative bound $d_{safe} = 2\vert P_e\vert +1$: EDZ$_{0.5}$ recovers close to BDZ throughput in finished jobs while still preserving all robots from entering DZ. However, EDZ$_{0.5}$ incurs a $3\%$ collision rate at $N=25$ with DZ size $1{\times}1$. Since no safety guarantee is established for $d = 0.5\vert P_e\vert $, these collisions empirically confirm that operating below the proven-safe separation $d_{safe} = 2\vert P_e\vert +1$ is insufficient to prevent collisions in practice. EDZ$_1$ maintains $0\%$ collision rate across all tested configurations, suggesting it may represent a practically useful middle ground, though its formal safety properties remain an open question. These results suggest a broader design space of escape strategies trading proven safety margins for throughput, of which EDZ$_{safe}$ and BDZ represent two specific instances.

The delay bound $\delta_i \leq \vert P_e \vert$ was satisfied in $100\%$ of EDZ$_{safe}$ runs across all configurations and DZ sizes, confirming Corollary~\ref{thm:collision_free_delay}. Both BDZ and EDZ$_{safe}$ are conservative in different ways: BDZ sacrifices affected robots, while EDZ$_{safe}$ sacrifices fleet density to guarantee their safety. The choice depends on operational priorities: BDZ is preferable when robot losses are acceptable and throughput is critical; EDZ$_{safe}$ is preferable when robot preservation is required. Future work should investigate tighter analytical bounds or problem-specific safety margins, as the EDZ$_1$ results suggest that collision-freedom may be achievable with significantly less conservative spacing.

Several modeling choices limit external validity: the discrete unit-speed grid abstraction, the assumption of negligible planning and communication latency ($\theta=1$), the restriction to a single rectangular dead zone, and the use of synthetic job sequences rather than real warehouse traces. 
This paper focuses on establishing the dead zone concept and its safety guarantees in the simplest setting sufficient to validate them. Extending the model along each of these dimensions is natural future work. However, note that our model extends to multiple consequent dead zones, given each one is announced after the latest robot affected by the previous one arrived at its destination. Each affected escaping robot would then receive a new job and with that a new trajectory that again satisfies $d_{safe}$. The previous dead zone, containing each affected dead robot, will then be treated as a static obstacle. Extending the model to multiple concurrent dead zones raises further questions. In our proofs, $|P_e|$ bounds how far a robot is allowed to deviate from its original trajectory in time and space while performing the escape maneuver without causing a collision. We call this bound robot's budget. If the robot encounters multiple dead zones on its path, we expect that it can safely continue as long as its accumulated deviation stays within this budget. However, if an escape has consumed the budget, a further escape is no longer covered by our collision-freedom guarantee. This direction remains to be explored. Dynamic obstacles unrelated to the dead zone (e.g. humans, forklifts) present a similar, orthogonal extension to our model. Reevaluation against real world warehouse data and physical robot kinematics would likewise be needed for deployment.

\section{Related Works}
\label{sec:related}
To position this work, we briefly survey the relevant literature on multi-agent path finding, fault tolerance, and safety in AGV fleets. Multi-agent path finding (MAPF) is the foundational problem of planning collision-free paths for a fleet of agents on a graph or grid~\cite{stern2019definitions, sharon2015cbs, stern2019mapf}. Offline centralized planning is one of the standard approaches in AGV-based systems~\cite{fragapane2021amr, okumura2023otimapp, okumura2021otimapp, chiddarwar2011multiAgent}, as it avoids runtime overhead and enables formal reasoning about system behavior. Numerous algorithms address the MAPF problem, ranging from Conflict-Based Search (CBS), a two-level search algorithm that resolves pairwise conflicts by iteratively adding constraints~\cite{sharon2015cbs}, to encodings based on Boolean Satisfiability (SAT) and Satisfiability Modulo Theories (SMT)~\cite{surynek2016efficientSAT, surynek2019unifying, surynek2020bounded}, as well as pickup-and-delivery extensions~\cite{liu2019taskPath} and warehouse applications~\cite{varambally2022mapf, salzman2020challenges, zhang2024planningExecution}. However, classical MAPF assumes normal operation throughout. Its variants that handle delays, faults, or uncertain obstacles are discussed below. Though, to our knowledge, none of those address what happens when part of the operational area becomes unavailable or unsafe.

Robustness to delays and agent failures has been studied in fault-tolerant MAPF. $k$-robust planning~\cite{atzmon2017krobust, chen2021symmetry} and MAPF with bounded edge traversal times~\cite{shahar2021safe} guarantee collision-freedom despite timing deviations. 
Okumura and Tixeuil~\cite{okumura2023faultTolerant} study crash faults in offline MAPF, where a crashed agent becomes a permanent point obstacle and healthy agents switch to a precomputed backup path at runtime. While closely related, their approach requires precomputing alternative plans, one for every possible crash pattern. Our dead zone instead defines a spatial area of unknown state (containing any anomaly, not only a robot crash), fenced off from ongoing operation; the corresponding handling with escaping requires only that robots keep a minimum distance from each other in their original trajectories, a single geometric rule checked once during planning. 

Park et al.~\cite{park2026robust} consider unexpected obstacles and replan only the affected robots online, replacing the global delay parameter of $k$-robust planning by per-segment margins; they do not bound the spatial deviation from the original trajectory and have no escape strategy comparable to ours. Fioravantes et al.~\cite{fioravantes2025agents} study a different problem: an individual agent unexpectedly pauses in place (a malfunction), and the question is how the rest of the fleet reacts, via decentralized, online, communication-based protocols. MAPF under obstacle uncertainty~\cite{shofer2023multi} computes contingency trees in advance for cells whose blocked state is discovered by sensing; these trees can grow exponentially in the number of uncertain cells, and scaling to many agents remains an open challenge. None of these works addresses the isolation of a region of unknown state. Moreover, unlike these approaches, which rely on online replanning, runtime protocols, or contingency trees, dead zone handling with escaping is a single static rule for which we prove collision-freedom. As the approaches discussed above solve different problems, we do not compare against them; we instead evaluate against a full emergency stop as a conservative reference mechanism, and Basic DZ Handling as the simplest sound alternative (Section~\ref{sec:evaluation}).

Formal collision and deadlock avoidance for AGV fleets during normal operation has been extensively studied~\cite{li2016collision, pratissoli2023hierarchical, lee2023geometric, chung2024deadlock, zhao2021sparezone, roselli2022compositional, witczak2019faulttolerant, desai2019safety}, including mechanized verification for a single ground robot amid adversarial obstacles~\cite{mitsch2017formal} and compositional verification for coordinating robot teams~\cite{majumdar2019motion}. However, these approaches focus on coordination under normal conditions as well as guaranteeing safety during ordinary task execution rather than spatial isolation of anomalies. Robot evacuation from spatially spreading hazards such as fire provides a conceptually related scenario~\cite{shih2024selfevacuation, guo2024agvfire}, but relies on online replanning and provides no formal collision-freedom guarantees for the remaining fleet. 
Prior work has shown that system-wide emergency stops cause secondary collisions~\cite{ogorelysheva2023mitigating}, and that localized fault handling is an open challenge~\cite{ogorelysheva2025anomaly, lienert2019failure, rodriguez2026disruptions, morais2025rfm}. 
Industry has begun addressing localized fault handling in practice~\cite{ssischaefer2026fastbots}, confirming real-world demand, yet to our knowledge no formal treatment exists. 


\section{Conclusion}
\label{sec:conclusion}
\vspace{-1px}
We introduced dead zones, a localized isolation for anomalies, as an additional fail-operational mechanism for the cases where a system-wide emergency stop is not strictly necessary, and defined two handling strategies for it: a basic one and one with escaping. Experimental evaluation against a system-wide emergency stop shows that basic dead zone handling reduces fleet downtime at the cost of losing affected robots, while dead zone handling with counterclockwise escaping saves them entirely --- with formal collision-freedom and bounded delay guarantees --- at the cost of a throughput penalty due to the minimum pairwise separation required between robots. This result constitutes a first step towards truly fault-tolerant AGV fleet operation, particularly critical in resource-constrained environments where losing a robot is not an option and a system-wide stop is not viable.
Future work will investigate more advanced escape strategies, including constraint programming and SMT/SAT-based approaches~\cite{surynek2019unifying, surynek2020bounded}, aiming for formal safety guarantees with less conservative robot separation requirements while reducing detour and delay overhead. 
Evaluations will be extended to more complex environments, including narrow corridors and other topological edge cases that stress-test the escape logic.

\section*{Acknowledgments}
This work has been funded by the Federal Ministry of Research, Technology and Space (BMFTR) via the 6GEM+ transfer hub under funding references 16KIS2412. The authors would also like to thank Mario Günzel for his valuable reviews, comments, and suggestions, which helped improve the quality of this paper.

\section*{AI Usage}
AI (Claude, Anthropic) was used for spell-checking, paraphrasing, and translation assistance during the writing of this paper. All scientific content --- including the dead zone model, handling strategies, definitions, theorems, proofs, literature research, and interpretation of the results --- was conducted without AI assistance. In the evaluation, implementation code was written with AI assistance; the overall architecture, metrics, and experimental design originate entirely from the authors, and all AI-written code was thoroughly reviewed and tested against the intended behavior before use.

\nocite{*}
\bibliographystyle{eptcs}
\bibliography{bibliography}
\end{document}